\documentclass[english,a4paper,12pt]{article}

\usepackage{amsxtra, amsfonts, amssymb, amstext, amsmath, mathtools}
\usepackage{amsthm}
\usepackage{booktabs}
\usepackage{nicefrac}
\usepackage{xspace}
\usepackage{authblk}

\usepackage{graphicx}

\graphicspath{{figures/}}

\usepackage{algorithm}
\usepackage[noend]{algpseudocode}
\usepackage{subcaption}

\usepackage{dsfont}
\usepackage[shortcuts]{extdash}
\usepackage{enumitem}
\usepackage{todonotes}

\usepackage{bbm}
\usepackage[utf8]{inputenc}
\usepackage{hyperref}
\usepackage[nameinlink]{cleveref}

\DeclarePairedDelimiter{\abs}{|}{|}
\DeclarePairedDelimiter{\set}{\{}{\}}

\allowdisplaybreaks

\newcommand{\NSGA}{NSGA\nobreakdash-II\xspace}

\newcommand{\oneminmax}{\textsc{OneMinMax}\xspace}
\newcommand{\leadingonestrailingzeros}{\textsc{LeadingOnesTrailingZeros}\xspace}

\newcommand{\omm}{\textsc{OMM}\xspace}

\newcommand{\lotz}{\textsc{LOTZ}\xspace}
\newcommand{\mlotz}{$m$\lotz}

\newcommand{\R}{\ensuremath{\mathbb{R}}}

\newcommand{\N}{\ensuremath{\mathbb{N}}} 

\let\originalleft\left
\let\originalright\right
\renewcommand{\left}{\mathopen{}\mathclose\bgroup\originalleft}
\renewcommand{\right}{\aftergroup\egroup\originalright}

\newtheorem{theorem}{Theorem}
\newtheorem{lemma}[theorem]{Lemma}

\begin{document}

\title{Difficulties of the NSGA-II with the Many-Objective LeadingOnes Problem}

\author[1]{Benjamin Doerr}
\author[2]{Dimitri Korkotashvili}
\author[1]{Martin~S. Krejca}

\affil[1]{Laboratoire d'Informatique (LIX), CNRS, École Polytechnique, Institut Polytechnique de Paris}
\affil[2]{École Polytechnique, Institut Polytechnique de Paris}

\date{}

\maketitle

\begin{abstract}
    The NSGA-II is the most prominent multi-objective evolutionary algorithm (cited more than 50,000 times). Very recently, a mathematical runtime analysis has proven that this algorithm can have enormous difficulties when the number of objectives is larger than two (Zheng, Doerr. IEEE Transactions on Evolutionary Computation (2024)). However, this result was shown only for the OneMinMax benchmark problem, which has the particularity that all solutions are on the Pareto front, a fact heavily exploited in the proof of this result.

    In this work, we show a comparable result for the LeadingOnes\-TrailingZeroes benchmark. This popular benchmark problem appears more natural in that most of its solutions are not on the Pareto front. With a careful analysis of the population dynamics of the NGSA-II optimizing this benchmark, we manage to show that when the population grows on the Pareto front, then it does so much faster by creating known Pareto optima than by spreading out on the Pareto front. Consequently, already when still a constant fraction of the Pareto front is unexplored, the crowding distance becomes the crucial selection mechanism, and thus the same problems arise as in the optimization of OneMinMax. With these and some further arguments, we show that the NSGA-II, with a population size by at most a constant factor larger than the Pareto front, cannot compute the Pareto front in less than exponential time.
\end{abstract}


\section{Introduction}

Many real-world optimization problems feature several conflicting objectives.
This results in incomparable optimal solutions---the \emph{Pareto optima}---, where each optimum that is better than another one in at least one objective is necessarily worse in another objective.
Solvers are therefore expected to return a rich and diverse set of Pareto optima, for discrete problems also all of them, collectively known as the \emph{Pareto front}.
Due to the hardness of many multi-objective problems, \emph{heuristic} algorithms are oftentimes employed and constitute an essential class of solvers in this field~\cite{ZhouQLZSZ11}.
The most prominent multi-objective heuristic is the \emph{non-dominated sorting genetic algorithm~II}~\cite{DebPAM02} (\NSGA), counting over 50\,000 citations on Google Scholar.

The \NSGA belongs to the successful line of \emph{multi-objective evolutionary algorithms}~\cite{ZhouQLZSZ11} (MOEAs).
As such, it maintains a multi-set (the \emph{population}) of~$N$ solutions and refines it iteratively, creating new solutions based on random mutations of existing ones and then selecting the best solutions found so far.
This results in complex dynamics, for which the first mathematical runtime guarantees were proven only recently~\cite{ZhengLD22}: The \NSGA optimizes the bi-objective \oneminmax~\cite{GielL10} (\omm) and \leadingonestrailingzeros~\cite{LaumannsTZ04} (\lotz) benchmarks of size~$n$ in, respectively, $O(N n \log n)$ and $O(N n^2)$ function evaluations (in expectation), when the population size $N$ is at least a constant factor larger than the size of the Pareto front.
These runtimes agree with the asymptotic runtimes of other MOEAs (with suitable parameter ranges), such as the (global) simple evolutionary multi-objective optimizer~\cite{LaumannsTZWD02,Thierens03,Giel03} ((G)SEMO), the $(\mu+1)$ SIBEA~\cite{BrockhoffFN08,NguyenSN15}, the decomposition-based multi-objective evolutionary algorithm~\cite{LiZZZ16} (MOEA/D; only for a \omm variant), the NSGA-III~\cite{WiethegerD23,OprisDNS24}, the SMS-EMOA~\cite{ZhengD24}, and the Strength Pareto Evolutionary Algorithm~2~\cite{RenBLQ24} (SPEA2).

Runtime analyses confirming the strength of the \NSGA have also been conducted in other directions, such as for combinatorial optimization~\cite{CerfDHKW23}, for noisy optimization~\cite{DangOSS23gecco}, with crossover~\cite{DangOSS23aaai,DoerrQ23crossover}, and in approximative settings~\cite{ZhengD24approx}.

All these positive result, however, regard bi-objective problems, and this is no coincidence. In a very recent result~\cite{ZhengD24many}, it was shown that the \NSGA fails to optimize the \omm benchmark efficiently once the number of objectives is~3 or more. More precisely, it was proven that the \NSGA with any population size that is linear in the size of the Pareto front for an exponential number of iterations cannot reach a population that witnesses the Pareto front.
The reason for this drastic decline in performance was attributed to the secondary criterion in the selection of the next population, the \emph{crowding distance}.
We note that a declining performance of \NSGA for a higher number of objectives has also been observed empirically~\cite{khare2003performance}, but this work does not exhibit such a drastic different between two and three objective, and it also does not blame the crowding distance for the difficulties.
That indeed the crowding distance is the likely culprit can now also be seen from the fact that the NSGA-III~\cite{WiethegerD23,OprisDNS24} and the SMS-EMOA~\cite{BianZLQ23,ZhengD24} provably easily optimize many classic benchmarks in three or more objectives; here we note that these two algorithms, as the \NSGA, use non-dominated sorting as first selection criterion, but use a different secondary selection criterion, namely closeness to reference points for the NSGA-III and the hypervolume for the SMS-EMOA.

In this light, the negative result of~\cite{ZhengD24many} appears to be a crucial step towards understanding the strengths and weaknesses of the different MOEAs massively used in practice. A clear limitation of this result, however, is the fact that it applies only to the \omm benchmark. This benchmark is extremely simple and in particular has the uncommon property that all solutions are Pareto-optimal, a fact heavily exploited in the proof of the result. While one could argue that proven difficulties on a simple benchmark should indicate even greater problems on more complex benchmarks, in view of the importance of this questions, it appears desirable to back up this intuition with more rigorous research. This is what we aim at in this work.

\textbf{Our contribution:}
We show that the poor performance of the \NSGA for more than~2 objectives is not restricted to \omm but extends to the \lotz benchmark. This benchmark does not have the particularity that all solutions are Pareto-optimal, rather the vast majority of the solutions are dominated by others. In that sense, this benchmark is closer to a real-world multi-objective optimization problem (while, of course, still being a synthetic benchmark simple enough to admit mathematical runtime analyses).

More precisely, we show the following result. Let the number of objectives $m$ be an even integer (the \lotz benchmark is only defined for even numbers of objectives). Assume that $m \ge 4$. Denote by $M = (\frac{2 n}{m} + 1)^{m / 2}$ the size of the Pareto front of the $m$-objective \lotz benchmark of problem size~$n$. Consider optimizing this benchmark via the \NSGA with population size at most $aM$, where $a > 1$ is an arbitrary constant. Then for all $T \in \N$, with probability at least $1 - T\exp(-\Omega(n))$ the \NSGA misses a constant fraction of the Pareto front for the first~$T$ iterations (see \Cref{thm:main,thm:general} and note that the precise failure probability of the result can be even smaller than the simplified $T \exp(-\Omega(n))$ expression stated here).
This result implies that the expected runtime (time to have the full Pareto front witnessed by the population) of the \NSGA on \lotz is at least exponential in expectation and with high probability when the number of objective is at least~$4$.

Although our main result is comparable to the one for the $m$-objective \omm benchmark by Zheng and Doerr~\cite{ZhengD24many}, due to the more complex fitness landscape of the \lotz problem, our analysis requires significant additional effort. In the \omm problem, all solutions are Pareto-optimal, and consequently, they are all contained in the first (and only) front of the non-dominated sorting of the combined parent and offspring population. Hence in this first front, a non-trivial selection decision is made via the crowding distance, and the deficiencies of the crowding distance detected in~\cite{ZhengD24many} lead to the undesired loss of Pareto-optimal solution values.

The \lotz benchmark, in contrast, has (strictly) dominating solution values. Consequently, it could happen that the first front of the non-dominated sorting contains at most $N$ solutions, and in this case, these (which include all Pareto optima in the population) would all survive into the next generation. Hence our main technical contribution is a careful analysis of the population dynamics, which shows that the size of the first front of the non-dominated sorting grows to above $N$ faster than that all Pareto optima are found; then, as in the case of \omm, a non-trivial selection decision is made in the first front and the short-comings of the crowding distance come into play.

We complement our theoretical results with an empirical analysis. Here we observe that the \NSGA is indeed capable of quickly finding a constant fraction of the Pareto front but struggles to get past a constant fraction.
Larger populations manage to cover larger fractions of the Pareto front.
In addition, we study random and binary-tournament parent selection mechanisms as well as uniform crossover and see that the performance is essentially the same.
Last, and perhaps most interestingly, we consider the \NSGA with one-bit mutation and see that while it still does not cover the entire Pareto front, it covers a much larger fraction than the algorithms with bit-wise mutation. We have no explanation for this.

\textbf{Outline:}
In \Cref{sec:preliminaries}, we introduce our notation and terminology and define the \NSGA as well as the $m$-objective \lotz benchmark (\mlotz).
We prove in \Cref{sec:theory} that the \NSGA is inefficient on \mlotz for $m \geq 4$, with \Cref{thm:main,thm:general} as our main result.
We complement our theoretical analyses empirically in \Cref{sec:experiments}, and we conclude our paper in \Cref{sec:conclusion}.

\section{Preliminaries}
\label{sec:preliminaries}

We start by introducing some general notation and terminology of multi-objective optimization and then define the $m$\lotz benchmark function and the \NSGA.

\subsection{Notation and Terminology}
\label{sec:preliminaries:notation}

Let~$\N$ denote the set of non-negative integers, hence including~$0$, and let~$\R$ denote the set of real numbers.
For all $k, \ell \in \N$, let $[k .. \ell] \coloneqq [k, \ell] \cap \N$ and $[k] \coloneqq [1 .. k]$.
Moreover, for all $k \in \N_{\geq 1}$,  $x \in \R^k$, and  $i \in [n]$, we denote the value of~$x$ at position~$i$ by~$x_i$.

In the following, let $n \in \N_{\geq 1}$ and $m \in \N_{\geq 3}$.
We consider \emph{pseudo-Boolean many-objective maximization}, that is, the maximization of functions $f\colon \{0, 1\}^n \to \R^m$.
We call each $x \in \{0, 1\}^n$ an \emph{individual}, we call $f(x) \in \R^m$ the \emph{objective-value of~$x$}, and for all $i \in [m]$, we denote the \emph{value of objective~$i$ of~$x$} by $f_i(x) \coloneqq f(x)_i$.

We compare all objective-values $u, v \in \R^m$ via \emph{domination}, which is a partial order.
We say that \emph{$u$ weakly dominates~$v$} (written $u \succeq v$) if and only if for all $i \in [m]$ we have that $u_i \geq v_i$, and we say that \emph{$u$ strictly dominates~$v$} (written $u \succ v$) if and only if at least one of these inequalities is strict.
If and only if neither~$u$ weakly dominates~$v$ nor the other way around, we say that~$u$ and~$v$ are \emph{incomparable}.
Moreover, we extend the notion of domination to individuals by implicitly referring to the domination of their objective-values.

For a given many-objective pseudo-Boolean function~$f$, we say that an individual $x \in \{0, 1\}^n$ is \emph{Pareto-optimal} if and only if for all $y \in \{0, 1\}^n$, we have that $f(y) \nsucc f(x)$; we call~$f(x)$ a \emph{Pareto optimum} in this case.
We call the set of all Pareto-optimal individuals the \emph{Pareto set (of~$f$)}, denoted by~$S^*$, and we call the set of all Pareto optima the \emph{Pareto front (of~$f$)}, denoted by~$F^*$.

When we use big-Oh notation, we always refer to the asymptotics in the length~$n$ of the bit-string representation, a number often called the \emph{problem size}.

\subsection{The Many-Objective LOTZ Benchmark}
\label{sec:preliminaries:mlotz}

Let $m \in \N_{\geq 4}$ be even, and let $n \in \N_{\geq m}$ such that~$m / 2$ divides~$n$.

\textbf{Definition.}
The $m$-objective \leadingonestrailingzeros~\cite{LaumannsTZ04} (\mlotz) benchmark bases its definition of the objectives on the bi-objective \lotz benchmark.
The latter maps an individual, as the name suggests, to the longest number of consecutive~$1$s from the first position and to the longest number of consecutive~$0$s from the last position.

In \mlotz, the~$m$ objectives are defined over $\frac{m}{2}$ independent parts of the bit-string, where each part defines two \lotz objectives for a different consecutive number of $n' \coloneqq \frac{2n}{m}$ bits of the input.
Formally, for all $k \in [m]$ and all $x \in \{0, 1\}^n$, abbreviating $f_k \coloneqq m\lotz_k$, we have
\begin{equation}
    \label{eq:mlotz}
    f_k(x) =
    \begin{cases}
        \sum_{i \in [n']} \prod_{j \in [i]} x_{j + (k-1)n'/2}           & \text{if } k \text{ is odd,} \\
        \sum_{i \in [n']} \prod_{j \in [i .. n']} (1 - x_{j+(k-2)n'/2}) & \text{else.}
    \end{cases}
\end{equation}
Consequently, $f_x(x)$ for odd~$k$ describes the number of leading~$1$s of the $(k+1)/2$-th block of the argument, whereas for even~$k$ it refers to the number of trailing~$0$s in the $k/2$-th substring.

\textbf{Pareto set and front.}
For all $i, j \in [n]$ with $i \leq j$, let $x_{[i .. j]}$ denote the individual $(x_{a})_{a \in [i .. j]}$, that is, the substring from position~$i$ to~$j$.
An individual $x \in \{0, 1\}^n$ is Pareto-optimal for \mlotz if and only if for all $k \in [\frac{m}{2}]$ there is an $i \in [0 .. n']$ such that $x_{[(k - 1) n' + 1 .. k n']} = 1^i 0^{n' - i}$.
Consequently, for a Pareto-optimal individual~$x$, we have for all $k \in [\frac{m}{2}]$ that $f_{2k - 1}(x) + f_{2k}(x) = n'$.
Thus, the size of the Pareto set and the Pareto front are both $(n' + 1)^{m / 2} \eqqcolon M$.

We say that two Pareto optima $u, v \in [0 .. n']^{m / 2}$ are \emph{neighbors} (denoted by $u \sim v$) if and only if there is an $i \in [\frac{m}{2}]$ such that $\abs{u_{2 i} - v_{2 i}} = 1$ (equivalent to $\abs{u_{2 i - 1} - v_{2 i - 1}} = 1$) and for all $j \in [\frac{m}{2}] \setminus \{i\}$ holds that $u_{2 j} = v_{2 j}$ (equivalent to $u_{2 j - 1} = v_{2 j - 1}$).

\subsection{The NSGA-II}
\label{sec:preliminaries:nsgaii}

The \emph{non-dominated sorting genetic algorithm II}~\cite{DebPAM02} (\mbox{\NSGA}, \Cref{nsga}) maximizes a given many-objective pseudo-Boolean function by iteratively refining a multi-set of $N \in \N_{\geq 1}$ individuals---the so-called \emph{population}.
It operates as follows, where we note that we define important operators below later.
The population is initialized with individuals drawn independently and uniformly at random.
Afterward, in an iterative fashion, from the current population $N$ new individuals are generated. We consider fair, random, and binary tournament for the selection of the \emph{parents}. As in most other theoretical works on the \NSGA, we do not consider here the use of crossover for the generation of the \emph{offspring}, but only mutation, namely \emph{standard bit mutation}.
Out of the combined parent and offspring population, the \NSGA selects~$N$ individuals based on a three-stage process.
First, the combined population is sorted via \emph{non-dominated sorting}, and at least~$N$ individuals are selected with respect to increasing fronts.
If more than~$N$ individuals are selected, ties among those in the highest front are broken with respect to the \emph{crowding distance} of the individuals, preferring individuals with a higher crowding distance.
If this still results in ties among individuals with the lowest crowding distance, those ties are broken uniformly at random.

\begin{algorithm}
    \caption{The non-dominated sorting genetic algorithm II~\cite{DebPAM02} (NSGA-II) for maximizing a given $m$-objective pseudo-Boolean function $f\colon \{0, 1\}^n \to \R^m$.}
    \label{nsga}
    \textbf{Input:} Population size $N \in \N_{\geq 1}$, objective function~$f$\\
    \textbf{Output:} Population $P_t$ upon termination
    \begin{algorithmic}[1]
        \State Independently and uniformly at random (u.a.r.), generate the initial population $P_0 = \{x_i\}_{i \in [N]}$
        \While{termination criterion not met}
        \State Generate the offspring population $Q_t$ of size $N$
        \State Combine populations $R_t = P_t \cup Q_t$
        \State Use non-dominated sorting to divide $R_t$ into fronts $F_1, F_2, \dots$
        \State Find the critical front~$i^*$
        \State Use \Cref{cdis} to calculate the crowding distance of each individual in $F_{i^*}$
        \State Let $F_{i^*}'$ be the $N - \left|\bigcup_{i \in [i^*-1]} F_i\right|$ individuals in $F_{i^*}$ with the largest crowding distance, breaking ties u.a.r.
        \State $P_{t+1} = \left(\bigcup_{i \in [i^*-1]} F_i\right) \cup F_{i^*}'$
        \State Increment~$t$ by~$1$
        \EndWhile
    \end{algorithmic}
\end{algorithm}

\textbf{Standard bit mutation.}
Given a \emph{parent} individual $x \in \{0, 1\}^n$, standard bit mutation creates a new \emph{offspring} individual $y \in \{0, 1\}^n$ by first copying $x$ and then flipping each bit independently with probability~$\frac{1}{n}$.
Formally, for all $i \in [n]$ independently, we have with probability $1 - \frac{1}{n}$ that $y_i = x_i$ and with  probability~$\frac{1}{n}$ that $y_i = 1 - x_i$.

\textbf{Non-dominated sorting.}
Given a population~$R$, non-dominated sorting divides the individuals recursively into~$I \in \N_{\geq 1}$ (population) fronts $(F_i)_{i \in [I]}$, depending on how many fronts dominate an individual.
For each $i \in I$, the individuals in~$F_i$ are those in~$R$ that are not strictly dominated if all individuals from smaller fronts are removed first.
Formally, $F_i \coloneqq \{x \in R \setminus \bigcup_{j \in [i - 1]} F_j \mid \forall y \in R \setminus \bigcup_{j \in [i - 1]} F_j\colon f(y) \nsucc f(x)\}$.

The \NSGA always selects full fronts increasingly until the number of individuals it selected is at least~$N$.
The smallest index that satisfies this property is called the \emph{critical front $i^* \in [I]$}.
That is, $i^* \coloneqq \min \set{i \in I \mid \abs{\bigcup_{j \in [i]} F_j} \geq N}$.
If and only if $\abs{\bigcup_{j \in [i^*]} F_j} > N$, then non-dominated sorting has selected too many individuals, and the \NSGA continues with removing individuals from~$F_{i^*}$ according to their crowding distance.

\textbf{Crowding distance.}
The crowding distance (\Cref{cdis}) of a population~$S$ assigns a non-negative number or infinity to each individual in~$S$.
This value is the sum of the crowding distances \emph{per objective}.
The latter is defined for each objective $i \in [m]$ by sorting~$S$ with respect to objective~$i$ (breaking ties arbitrarily) and assigning a crowding distance for~$i$ of infinity to the two end points of this sorting.
The crowding distance for~$i$ of all other individuals is the difference of objective~$i$ of their two neighbors divided by the difference of objective~$i$ of the two end points (\cref{line:crowdingDistanceComputation}).

The \NSGA selects among the individuals in~$F_{i^*}$ those $N - \abs{\bigcup_{i \in [i^* - 1]} F_i}$ with the largest crowding distance.
Ties are broken uniformly at random.

\begin{algorithm}
    \caption{Computation of the crowding distance $\text{cDis}(S)$.}
    \label{cdis}
    \textbf{Input:} $S = \{S_i\}_{i \in [\abs{S}]}$, a multi-set of individuals \\
    \textbf{Output:} $\text{cDis}(S) = (\text{cDis}(S_i))_{i \in [\abs{S}]}$, where $\text{cDis}(S_i)$ is the crowding distance for $S_i$
    \begin{algorithmic}[1]
        \State $\text{cDis}(S) = (0)_{i \in [\abs{S}]}$
        \For {each objective $f_i$}
        \State Sort $S$ in order of descending $f_i$ value: $S_{i,1}$ to $S_{i,|S|}$
        \State $\text{cDis}(S_{i,1}) = +\infty, \text{cDis}(S_{i,|S|}) = +\infty$
        \For {$j = 2$ to $|S| - 1$}
        \State $\text{cDis}(S_{i,j}) = \text{cDis}(S_{i,j}) + \frac{f_i(S_{i,j-1}) - f_i(S_{i,j+1})}{f_i(S_{i,1}) - f_i(S_{i,|S|})}$
        \label{line:crowdingDistanceComputation}
        \EndFor
        \EndFor
    \end{algorithmic}
\end{algorithm}

\textbf{Algorithm-specific terminology.}
We say that a population~$P$ \emph{covers} a Pareto optimum $v \in \R^m$ (of~$f$) if and only if there is an individual $x \in P$ such that $f(x) = v$.
We extend this notion to sets~$F$ of Pareto optima, saying that \emph{$P$ covers~$F$} if and only if each $v \in F$ is covered by~$P$.

The \emph{runtime}~$T$ of the \NSGA (on an $m$-objective function) is the (random) number of evaluations of~$f$ until the population of the \NSGA covers the Pareto front~$F^*$ of~$f$ for the first time.
Formally, since the \NSGA performs~$N$ evaluations of~$f$ in each iteration (and~$N$ once initially), it holds that $T = \min \set{N(t + 1) \in \N \mid P_t \textrm{ covers } F^*}$.

\section{Theoretical Runtime Analysis}
\label{sec:theory}

Our main result in this section is \Cref{thm:main} (for fair parent selection) and its extension to random and binary tournament selection (\Cref{thm:general}), which prove that the \NSGA (\Cref{nsga}) struggles to optimize \mlotz (\cref{eq:mlotz}) for $m \in \N_{\geq 4}$ objectives within sub-exponential time if its population size~$N$ is linear in the size~$M$ of the Pareto front.
We recall from \Cref{sec:preliminaries:mlotz} that $M = (\frac{2 n}{m} + 1)^{m / 2}$.
The result shows that the \NSGA is missing a constant fraction of the Pareto front with overwhelming probability.
That is, the failure probability is at least exponentially small in $\frac{m}{n} M$ in each iteration.
Thus, for any number of iterations sub-exponential in $\frac{m}{n} M$, the \NSGA is missing a constant fraction of the Pareto front with probability $1 - o(1)$.

\begin{theorem}\label{thm:main}
    Let $m \in \mathbb{Z}_{\geq 4}$ and $a > 1$ be constants, with $m$ even. Consider using the NSGA-II with population size  $N \leq aM = a(2n/m + 1)^{m/2}$, fair selection (every parent creates one offspring), and bitwise mutation to optimize \mlotz with problem size $n$. Then for any $T$ and all $1 \leq t \leq T$, with a probability of at least $1 - T\exp(-\Omega(\frac{m}{n}M))$ the combined parent and offspring population $R_t = P_t \cup Q_t$ does not cover a constant fraction of the Pareto front.
\end{theorem}

In all of the remainder, we assume that the assumptions of the theorem above are fulfilled.

From~\cite{ZhengD24many}, we obtain the following result, which captures the key short-coming of the crowding distance. This result was proven as Lemma~1 in~\cite{ZhengD24many} only for the optimization of the \omm benchmark, but a closer inspection of the proof reveals that it is valid for any many-objective problem in which each of the objectives takes at most $n+1$ different values.

\begin{lemma}\label{pos}\label{lem:pos}
    Let $S$ be a set of pair-wise non-dominated individuals in $\{0,1\}^n$. Assume that we compute the crowding distance $\text{cDis}(S)$ concerning the objective function $m$\textsc{LOTZ} via \Cref{cdis}. Then at most $4n + 2m$ individuals in $S$ have a positive crowding distance.
\end{lemma}

The next lemma quantifies how the random selection among individuals with crowding distance zero leads to the loss of Pareto-optimal solution values. The proof of this result is very similar to the main arguments in the proof Theorem 4 in \cite{ZhengD24many}. In fact, our lemma generalizes the central argument in that theorem. In the case of the OneMinMax problem, all solutions belong to the Pareto front, resulting in a single front $F_1$ with size $|F_1| = (1 + \alpha)N > N$ with $\alpha = 1$. For our main result, we need to understand the general situation where $F_1$ has cardinality at least $(1+\alpha)$, where $\alpha$ can be any positive constant.

\begin{lemma}[Random selection lemma]\label{fraction}\label{lem:fraction}
    Let  $N \le aM$. Suppose that at some iteration $t$ in the combined population $R_t = P_t \cup Q_t$, we have $|F_1| > (1 + \alpha)N$ for some $\alpha > 0$. Then, in the population $P_{t+1}$ with an $1 - \exp\left(-\Omega(M)\right)$ probability we will be missing at least a $d(\alpha) = \frac{1}{6}\left(\frac{\alpha}{2e} + \frac{1 - 8a}{2eN}\right)^{8a}$ fraction of the Pareto front.
\end{lemma}

\begin{proof}
    Since we are aiming at an asymptotic statement, we can and shall assume that $n$, and hence $M$, are sufficiently large. Let $U^* = F^* \cap f(R_t) = \{u \in F^* \mid \exists x\in R_t, f(x)=u\}$ and let $U'$ be the set of points in the Pareto front that have more than $\Delta\coloneqq 8a$ individuals in $R_t$ corresponding to it, i.e., $U' = \{u \in U^* \mid |\{x \in R_t \mid f(x) = u\}| \geq \Delta \}$.
    Notice that if $|U^*| < 3M/4$, then we are missing at least $M/4 > d(\alpha)M$ points from the Pareto front. Hence we may assume that $|U^*| \geq 3M/4$.
    We see that $\Delta|U'| \leq |R_t| \leq 2N < 2aM$, thus $|U'| \leq \frac{2aM}{\Delta} \leq \frac{M}{4}$, and we have
    $$|U^* \setminus U'| \geq |U^*| - \frac{M}{4} \geq \frac M2.$$

    Let $U = \{u \in U^{*} \setminus U' \mid \forall x : f(x) = u \Rightarrow \text{cDis}(x) = 0\}$. Let us further denote by $A$ the number of individuals in $F_1$ that have crowding distance greater than zero. By \Cref{pos} we have $A \leq 4n + 2m$ and, taking $n$ large enough, we obtain
    $$|U| \geq |U^{*} \setminus U'| - A \geq \frac{M}{2} - 4n - 2m \geq \frac{M}{3}.$$

    Now our goal is to compute the probability that a value $u \in U$ is not covered by~$P_{t+1}$. For $u$, all the individuals corresponding to it have the crowding distance equal to $0$, and every such individual is in the $F_1$. This means, to form $P_{t+1}$ we first have to select $A$ individuals that have positive crowding distance and then select $N - A$ individuals from the remaining $|F_1| - A$ uniformly randomly.
    Notice that this is the same as randomly picking $N'\coloneqq |F_1| - N$ individuals with zero crowding distance from $|F_1|$ and removing them. Let $X$ be the random variable modeling the number of uncovered points in $P_{t+1}$. Our goal is to bound $\Pr\left(X \leq \text{threshold}\right)$ for a suitable threshold, which will be introduced later. To ease the argument, we instead consider a process where  $N'$ times independently and uniformly at random we mark an individual from $F_1$ with zero crowding distance. We then remove from $F_1$ all individuals marked at least once. Let us denote by $Y$ the random variable modeling the number of uncovered points in this process. Then $X$ stochastically dominates $Y$, simply because in the original process we remove at least as many individuals as in the new process, and in either process, the set of individuals removed is uniformly distributed among all sets of the respective cardinality. We hence regard the process leading to $Y$ in the following.

    Let $u \in U$ and assume that there are exactly $b$ corresponding individuals in $R_t$. For $u$ to be uncovered in the next generation, we need all of its corresponding individuals from $R_t$ to occur at some point during the $N'$ picks. By only regarding events in which each of the $b$ individuals is picked exactly once,
    we obtain that the probability of $u$ being uncovered is at least
    \begin{align*}
        \binom{N'}{b} & \frac{b!}{(|F_1|-A)^b} \left(1 - \frac{b}{|F_1|-A}\right)^{N'-b}                                                           \\
                      & = \frac{N'!}{(N'-b)!(|F_1|-A)^b}                                                                                           \\
                      & \quad \cdot \left(1 - \frac{b}{|F_1|-A}\right)^{\left(\frac{|F_1|-A}{b} - 1\right) \left(\frac{(N'-b)b}{|F_1|-A-b}\right)} \\
                      & \geq \frac{N'!}{(N'-b)!(|F_1|-A)^b} \exp \left(-\frac{(N'-b)b}{|F_1|-A-b}\right)                                           \\
                      & \geq \frac{N'!}{(N'-b)!e^b(|F_1|-A)^b} = \frac{N' \cdots (N'-b+1)}{e^b(|F_1|-A)^b}                                         \\
                      & \geq \left(\frac{N'-b+1}{e(|F_1|-A)}\right)^b  \geq \left( \frac{\alpha N - \Delta + 1}{e(|F_1|-A)} \right) ^{\Delta},
    \end{align*}
    where in the first inequality we used that for $n$ big enough $b \leq \Delta \leq |F_1|-A$ and $(1 - 1/r) ^ {r-1} \geq \exp(-1)$ for all $r>1$. The second inequality follows from $|F_1| - A \geq N' = |F_1| - N$, which is true, as for $n$ big enough $N > A$. The final inequality is due to $b \leq \Delta$. We conclude that
    \begin{align*}
        \Pr[u \text{ not covered}]
         & \geq \left( \frac{\alpha N - \Delta + 1}{e(|F_1| - A)} \right)^{\Delta} \geq \left( \frac{\alpha N - \Delta + 1}{2eN} \right)^{\Delta} \\
         & \geq \left( \frac{\alpha}{2e} + \frac{1 - 8a}{2eN} \right)^{8a} \coloneqq p.
    \end{align*}

    Since each point in $U$, with a probability of at least $p$, it is not covered by $P_{t+1}$, we have $E [Y] \geq p|U| \geq p\frac{M}{3}$. As $Y$ can be written as a function of $N' \geq \alpha N$ independent picks, by applying McDiarmid's \emph{method of bounded differences}~\cite{McDiarmid98} (also found as Theorem 1.10.27 in~\cite{Doerr20bookchapter}), we obtain

    \begin{align*}
        \Pr \left[Y \leq \frac{1}{6} p M \right]
         & \leq \Pr \left[Y \leq \frac{1}{2} E[Y] \right]                     \\
         & \leq \exp\left(-\Omega(E[Y])\right) = \exp\left(-\Omega(M)\right).
    \end{align*}

    As $X$ stochastically dominates $Y$, we deduce $\Pr \left[X \leq \frac{1}{6} p M \right] \leq \Pr \left[Y \leq \frac{1}{6} p M \right] \leq \exp\left(-\Omega(M)\right)$. This means with the probability of $1 - \exp\left(-\Omega(M)\right)$ we will be missing at least $\frac{1}{6} p M = d(\alpha)M$ points from the Pareto front.
\end{proof}

We define some additional notation needed for the proof. We first specify that we always treat sets of solutions as multi-sets, that is, we allow them to contain the same individual multiple times. In contrast, sets of objective values are always seen as proper sets, that is, they contain each objective vector at most once. Let $P_t^*$ be the (multi-)set of points defined by $P_t^{*} = \{x \in P_t \mid f(x) \in F^* \}$. Define $N_t^* \coloneqq |P_t^{*}|$, the number of Pareto optima in $P_t$, and $M_t' \coloneqq |f(P_t) \cap F^*|$, the number of different objective values of solutions on the Pareto front.

As outlined  in the introduction, the main additional challenges over the analysis of the \omm case in~\cite{ZhengD24many} is that we need to understand how the population is distributed over the fronts of the non-dominated sorting. More specifically, the short-coming of the crowding distance can only come into play when the crowding distance plays a role in the selection of solutions on the Pareto front. For this to happen, we need to have more than $N$ solutions in the first front of the non-dominated sorting.

\Cref{lm2} below shows that we arrive at such a situation, and this before having found the full Pareto front. The proof of this result is a careful analysis of the population dynamics, exhibiting that we simultaneously have a multiplicative increase of the number of solutions on the Pareto front (and these necessarily end up in the first front), but only a very slow growth of the number of different solutions on the Pareto front. This leads to a situation with the first front having at least $(1+\alpha) N$ elements without ever having witnessed more than $(1-\beta)N$ different solutions on the Pareto front, where $\alpha, \beta$ are suitable positive constants. In this situation, we can apply the random selection lemma (\Cref{fraction}) and show that we lose sufficiently many solutions on the Pareto front to iterate this argument.

\begin{lemma}\label{lm2}\label{lem:main}
    Let $c_1 \coloneqq 1 - d(\frac{1}{12ae})$ and $c_2 \coloneqq (1 + c_1)/2$. Assume that $M/2 \leq M_{t}' \leq c_1M$ for some $t$. Then there exists a $T = O(n)$ such that, with probability $1 - (T+1) \exp(-\Omega(mM/n))$, for all $0\leq i \leq T$, we have $M_{t + i}' \leq c_2M$ and $M_{t + T + 1}' \leq c_1M$.
\end{lemma}

To prove this lemma, we first show that $M'_t$, the number of different Pareto optima in the population, cannot increase by a lot in a single iteration. We shall need this argument again in the final proof of our main result.

\begin{lemma}\label{lem:growth}
    Let $t$ be some iteration. Then
    \[
        \Pr[f(Q_t \cup P_t) \cap F^* \ge M'_t + \tfrac{2amM}{n}] \leq \exp(-\tfrac{maM}{3n})).
    \]
    In particular,
    \[
        \Pr[M'_{t+1} \ge M'_t + \tfrac{2amM}{n}] \leq \exp(-\tfrac{maM}{3n})).
    \]
\end{lemma}

\begin{proof}
    The second statement is an immediate consequence of the first one since any new Pareto optimum in $P_{t+1}$ must be contained in $Q_t$. We thus prove the first statement.

    Let $x \in P_t$ and let us denote by $x' \in Q_t$ its offspring (recall that we assume fair selection, that is, each parent created one offspring). The probability that $f(x')$ is in $F^*$, but not in $f(P_t)$, is at most the probability that $x'$ is a Pareto set point different from $x$. We first show that this probability it at most $\frac{m}{n}$. We know that there are $m$ objectives and $m' =m/2$ blocks of leading ones and trailing zeroes. Let us denote by $\ell_i = \min \{ k \in \mathbb{N} \mid x_k = 0 \text{ and } \frac{(i-1)n}{m'} \leq k \leq \frac{in}{m'} \}$ and $r_i = \max \{ k \in \mathbb{N} \mid x_k = 1 \text{ and } \frac{(i-1)n}{m'} \leq k \leq \frac{in}{m'} \}$, the positions of the first $0$ from the left and the first $1$ from the right in the $i$-th block. For $x'$ to be a Pareto set point different from~$x$,  at least one of the bits $\{x_{\ell_1}, x_{r_1}, \ldots, x_{\ell_{m'}}, x_{r_{m'}}\}$ has to be flipped. The probability for any bit in $x$ to be flipped is~$\frac{1}{n}$. Thus using the union bound, we easily conclude that the probability of $x'$ being a different Pareto set point is at most $\frac{2m'}{n} = \frac{m}{n}$.

    Let $X$ denote the random variable that counts the number of offspring that are in the Pareto set and are different from their parent. Then $X = \sum_{u \in P_t} X_u$, where $X_u$ is a binary random variable that, for each $u \in P_t$, models whether the offspring is a Pareto set point different from the parent. As we just computed that $E[X_u] \leq \frac{m}{n}$, from $N \leq aM$ we obtain $E[X] \leq \frac{maM}{n}$. Since the $X_u \in [0, 1]$ are independent, using a multiplicative Chernoff bound working with estimates for the expectation (e.g., Theorem~1.10.21 from \cite{Doerr20bookchapter}) we have $\Pr[X \geq \frac{2amM}{n}] \leq \exp(-\frac{maM}{3n})$. Since $f(Q_t \cup P_t) \cap F^* \le M'_t + X$, we have shown our claim.
\end{proof}

We are now ready to prove \Cref{lm2}, the heart of our analysis.

\begin{proof}[Proof of \Cref{lm2}]
    The proof of this lemma is divided into three parts. The first part, exploiting \Cref{lem:growth}, shows that, starting with the condition of the lemma, covering more than $c_2M$ different points in the Pareto front requires more than $\frac{(c_2 - c_1)n}{2am}$ steps.  In the second part we prove that in $\frac{(c_2 - c_1)n}{2am}$ steps, there will be a point when $|F_1| \geq (1 + \alpha) N$, where $\alpha \coloneqq \frac{1}{12ae}$. The last part uses the \Cref{fraction} to conclude that as result of the next iteration, we will indeed be missing at least $c_1 M$ points of the Pareto front.

    \textbf{Part 1:} By \Cref{lem:growth}, with probability at least $1 - \exp(-\Omega(\frac{m}{n}M))$, we cover at most $\frac{2amM}{n}$ new points in the Pareto set in one iteration. To cover more than $c_2M$ points in the Pareto front, we need to obtain more than $(c_2 - c_1)M$ new points, therefore the algorithm will require more than $\frac{(c_2 - c_1)n}{2am}$ iterations, with probability at least $1 - \frac{(c_2 - c_1)n}{2am} \exp(-\Omega(\frac{m}{n}M)) = 1 - \exp(-\Omega(\frac{m}{n}M))$.

    \textbf{Part 2:} Let us denote by $F_1^k$ the non-dominated points in $R_k$ and for simplicity let $\delta \coloneqq \frac{(c_2 - c_1)}{2am}$. Aiming at a contradiction, assume that for all $k \leq t + \delta n$ and for $\alpha \coloneqq \frac{1}{12ae}$, we have $|F_1^k| < (1 + \alpha)N$.

    We note that for all $x \in P_t^*$, the probability that we do not flip any bit of $x$ is $\left( 1 - \frac{1}{n} \right)^n$, as each bit stays the same with probability $1 - \frac{1}{n}$. Hence for $n \geq 3$, the generated offspring $x'$ equals the parent with probability at least $\left( 1 - \frac{1}{n} \right) \left( 1 - \frac{1}{n} \right)^{n-1} > \frac{n-1}{ne} \ge \frac{2}{3e}$.

    Let $O'\coloneqq f(P_t) \cap F^*$ be the set of Pareto front points that are already covered by $P_t$ at the start of the process. For $k$ such that $t \leq k \leq t + \delta n$, let $O_k = \{x \in P_k \mid f(x) \in O' \}$. Our goal is to show that for all $k$, we have $|O_{k}| \geq |O_{t}|(1 + \frac{1}{6e})^{k - t}$, apart from a failure probability of at most $\exp(-\Omega(M))$ for each $k$. Before conducting this proof, we mention that from the assumptions of the lemma, we have $|O_t| \geq M_t' \geq M/2$.

    We will show our claim using induction on $k$. The base case is trivial, so we proceed to the induction step. Assume that for some $k$ we have $|O_{k}| \geq |O_{t}|(1 + \frac{1}{6e})^{k - t}$. Let $X$ be the random variable that counts the number of created copies of the elements of $O_k$, i.e., the offspring of the elements of $O_k$ that are the same as their parents. It is crucial to point out that $X$ just models the number of created copies, not the copies that are selected in the next generation. Each point $x \in O_k$ creates a copy with a probability of at least~$\frac{2}{3e}$, hence $E[X] \geq \frac{2}{3e}|O_k|$. Again, we can write $X = X_1 + \ldots + X_{|O_k|}$, where $X_i$ is a binary random variable modeling the creation of a copy of the $i$-th point in the set $O_k$. Since the $X_i$ are bounded and independent, using a Chernoff bound we deduce that
    \begin{align*}
        \Pr \left[ X \leq \frac{e}{3} \right]
         & \leq \Pr \left[ X \leq \frac{E[X]}{2} \right]                      \\
         & \leq \exp\left(-\Omega(E[X])\right) = \exp\left(-\Omega(M)\right).
    \end{align*}

    As $|F_1^k| \leq (1 + \alpha)N$, and as for $P_{k+1}$ we select $\min\{N, |F_1^k|\}$ points from $F_1^k$, with a probability of at least $1 -  \exp\left(-\Omega(M)\right)$, we have
    \begin{align*}
        |O_{k+1}|
         & \geq |O_k|\left(1 + \frac{1}{3e}\right) - \alpha N                \\
         & \geq |O_k|\left(1 + \frac{1}{3e}\right) - \frac{1}{6e}\frac{M}{2} \\
         & \geq |O_k|\left(1 + \frac{1}{6e}\right),
    \end{align*}
    where in the second inequality, we used $N \leq aM$ and $\alpha = \frac{1}{12ae}$, and the last inequality follows from $M/2 \leq |O_1| \leq |O_k|$. Therefore, using the induction hypothesis, we can deduce that $|O_{k+1}| \geq |O_t|\left(1 + \frac{1}{6e}\right)^{k + 1 - t} \geq \left(1 + \frac{1}{6e}\right)^{k + 1 - t}$. To conclude the argument, notice that for $n$ big enough, more precisely, $n > \frac{\ln (2(1 + \alpha)a)}{\delta \ln(1 + 1/6e)}$, we have $|F_1^k| \geq |O_{t + \delta n}| \geq \frac{M}{2}\left(1 + \frac{1}{6e}\right)^{\delta n} > (1 + \alpha)aM \geq (1 + \alpha)N$, and we obtain a contradiction.

    \textbf{Part 3:} In the first two parts of this proof, we have seen that there is a $T \le \delta n$ such that in the time interval $[t..t+T]$, we have never covered more than $c_2 M$ points of the Pareto front, and such that in $R_{t+T} = P_{t+T} \cup Q_{t+T}$, we have $|F_1| \geq (1 + \alpha)N$. This statement holds apart from a failure probability of at most $T \exp(-\Omega(\frac{m}{n}M))$. From \Cref{fraction}, we now conclude that $P_{t+T+1}$ is missing at least a $d(\alpha) \geq (1 - c_1)$ fraction of the Pareto front, with a probability of at least $1 - \exp(-\Omega(M))$.
\end{proof}

The proof of our main result now is, mostly, a repeated application of the lemma just shown.

\begin{proof}[Proof of \Cref{thm:main}]
    Consider a run of the \NSGA on the \mlotz benchmark. We first show that with high probability, the random initial population satisfies one of the conditions of \Cref{lm2}, namely that it covers at most $c_1 M$ points of the Pareto front, that is, we have $M'_0 \le c_1 M$. To this aim, we note that a random individual is in the  Pareto set with probability $|F^*|/2^n = M/2^n$, where we recall that there is a one-to-one correspondence between the Pareto front and the Pareto set in the \mlotz benchmark. Consequently, the expected number of Pareto set points in $P_0$ is $\frac{NM}{2^n} \le M/4$ when assuming $n$ to be sufficiently large. Using a Chernoff bounds that works relative to an estimate for the expectation, e.g., Theorem~1.10.21 from \cite{Doerr20bookchapter}, we conclude that $\Pr[M'_0 < c_1 M] \ge \Pr[M'_0 < M/2] \ge 1 - \Pr[M_0' \ge 2 (M/4)] \ge 1 - \exp(-M/12)$.

    We next show that if at some time $t$ we have $M'_t \le c_1 M$, then there is a time $s > t$, $s = t + O(n)$ such that, with probability at least $1 - (s-t)\exp(-\Omega(\frac{mM}{n}))$, we have $M'_r \le c_2 M$ for all $r \in [t..s-1]$ and furthermore $M_s \le c_1 M$. We argue differently depending on $M'_t$.

    If $M_{t}' \leq M/2$, then \Cref{lem:growth} immediately implies that $s = t+1$ fulfills our claim.

    Hence assume that we have ${M}/{2} \leq M'_t \le c_1 M$. By \Cref{lm2} we know that there exists $T$ such that with  probability at least $1 - (T+1) \exp(-\Omega(\frac{mM}{n}))$  we have $M'_{t + T + 1} \leq c_1 M$ and for all $k \in [t..t + T]$ we cover at most $M'_k \leq c_2 M$ points from the Pareto front. Hence  $s = t + T + 1$ fulfills our claim.

    Via induction, using a union bound to add up the failure probabilities, we obtain the statement that for all $t$ there is an $s = t+O(n)$ such that with probability $1 - s\exp(-\Omega(\frac{mM}{n}))$, the parent population of the \NSGA optimizing \mlotz has, in the first $t$ iterations, never covered more than $(1-c_2)M$ points of the Pareto front. Noting that $\frac{mM}{n} = \Omega(n)$ and adjusting the implicit constants in the asymptotic notation, we also have that for all $t$, with probability $1 - t\exp(-\Omega(\frac{mM}{n}))$ we have, in the first $t$ iterations, never covered more than  $(1-c_2)M$ points of the Pareto front.

    So far we have proven that $P_t$ misses a constant fraction of the Pareto front for a long time. By the first statement of \Cref{lem:growth}, the number $f(P_t \cup Q_t) \cap F^*$ of Pareto front points covered by the combined parent and offspring population is at most by a lower-order term larger than $f(P_t) \cap F^*$. Hence our result immediately extends to the combined parent and offspring population, and this proves our theorem.
\end{proof}

We now extend our main result to two other variants of the \NSGA, namely the ones obtained by replacing fair selection by random selection or (independent) binary tournament selection. Both are common variants of the \NSGA. To profit from similarities of the two selection methods, we prove a result for a broader class of selection operators that includes these two.

\begin{theorem}\label{thm:general}
    The statement of \Cref{thm:main} remains valid if the \NSGA selects the individuals to be mutated in any way that ensures that (i)~each parent individual is chosen independently, and (ii) in a way that there is a constant $\gamma > 0$ such that for any set $O$ of individuals on the Pareto front with $|O| \ge M/2$, the parent individual is from this set with probability at least $\gamma$. This includes in particular the cases that each parent is chosen uniformly at random from the population (random selection) and that each parent is chosen by selecting two random individuals from $P_t$ and then taking as parent the one that appears earlier in the non-dominated sorting, using the crowding distance as tie-breaker (binary tournament selection).
\end{theorem}

\begin{proof}
    We note that \Cref{lem:pos,lem:fraction} only regard the selection of the next population, so they are independent of the way the offspring are generated.

    The two main arguments in the proof of \Cref{lem:growth} are that an offspring with probability at most $2m'/n$ leads to a new Pareto-optimal solution value, and that this, with the help of a Chernoff bound, can be used to prove a probabilistic upper bound on the expansion of the population on the Pareto front. For the first argument, no particular assumption was made on the parent, hence it immediately extends to other parent selection operators. For the Chernoff bound, we have to organize ourselves a little differently (since now not each individual in $P_t$ gives rise to an offspring), but the remaining calculations remain unchanged. Assume that the $N$ offspring are generated sequentially. For $i \in [1..N]$ let $X_i$ be the indicator random variable of the event that the $i$-th offspring lies on the Pareto front and is different from its parent. These random variables are independent and, by the first argument, satisfy $\Pr[X_i = 1] \le 2m'/n = m/n$. Hence the sum $X = \sum_{i=1}^N X_i$, which is an upper bound on the number of newly generated Pareto-optimal solution values admits the same Chernoff bounds as the corresponding sum $X$ in the proof of \Cref{lem:growth}.

    For the proof of \Cref{lem:main}, we first note that the first part only builds on \Cref{lem:growth} without exploiting directly particular properties of the algorithm. In the second part of the proof, we argue that we generate quickly new solutions (with existing solution values) on the Pareto front. Again, since we now select parents independently, we have to reorganize the central Chernoff bound, but more importantly, we have to ensure that the right parent individuals are selected sufficiently often. To this aim, we recall in the second part of the proof of \Cref{lem:main} we assume that the number $|O'|$ of points of the Pareto front covered at the start of this phase is at least $M/2$. Consequently, the set $O_k$ defined in this proof always has a cardinality at least $M/2$. By our assumptions, this implies that each parent chosen in iteration $k \in [t..t+\delta n]$ is in $O_k$ with probability at least some positive constant $\gamma$. Hence the probability that the $i$-th offspring generated in this iteration is equal to its parent, which is in $O_k$, is at least $\gamma (1-\frac 1n)^n \ge \frac{2\gamma}{3e}$. Denoting by $X_i$ the indicator random variable for this event, we note that $X = \sum_{i=1}^N X_i$ counts the number of offspring equal to their parent in $O_k$. Using the same Chernoff bound as in \Cref{lem:main} and replacing the constant $\alpha$ by $\gamma \alpha$, we obtain that with probability at least $1 - \exp(-\Omega(M))$, we have $|O_{k+1}| \ge |O_k|(1 + \frac{\gamma}{6e})$. By adjusting also the other constants in \Cref{lem:main}, we have shown a statement differing from this lemma only in the constants $c_1$, $c_2$, and the implicit constants in the asymptotic notation.

    Finally, the main proof of \Cref{thm:main} only uses the previously shown lemmas, so we have shown our claim, which differs from the claim of \Cref{thm:main} only in the constants hidden in the asymptotic notation.

    We finally argue that random selection and binary tournament selection are covered by this theorem. Let $O$ be any set of individuals such that $|O| \ge M/2$. Recall that $N \le a M$ for some constant. For random selection, the probability to pick an individual from $O$ is apparently $\frac{|O|}{N} \ge \frac{M/2}{aM} = \frac{1}{2a} \eqqcolon \gamma$. For binary tournament selection, the probability to select an individual from $O$ is at least the probability that both individuals participating in the tournament are in $O$, which is $(\frac{|O|}{N})^2 \ge (\frac{M/2}{aM})^2 = \frac{1}{4a^2} \eqqcolon \gamma$.
\end{proof}

We note that our proof method does not cover all variants of the \NSGA regarded so far. In particular, surprisingly, our proof does not extend to one-bit mutation. The reason is that the argument that the number of individuals in the Pareto front grows quickly in the second part of the proof of \Cref{lem:main} breaks down, since with one-bit mutation, we never create a copy of a parent.
In \Cref{sec:experiments}, we investigate the impact of one-bit mutation empirically, and \Cref{fig:results_A_B} suggests that while the algorithm still does not manage to cover the entire Pareto front in reasonable time, it covers a larger fraction than with standard bit mutation.

Our proof also does not immediately extend to crossover. The bottleneck here is that it is challenging to show that the probability of generating a new Pareto-optimal solution value is small, which is the main argument in the proof of \Cref{lem:growth}. We note that there are very few proofs of lower bounds for crossover-based algorithms, again for the reason that the population dynamics are harder to analyze.
We briefly study this setting empirically though in \Cref{sec:experiments}.
\Cref{fig:coverage_combined_population} shows that adding crossover has no real impact on the suboptimal coverage of the algorithm.

\section{Experiments}
\label{sec:experiments}

\begin{figure}[t]
    \centering
    \begin{minipage}{0.95\columnwidth}
        \includegraphics[width=\linewidth]{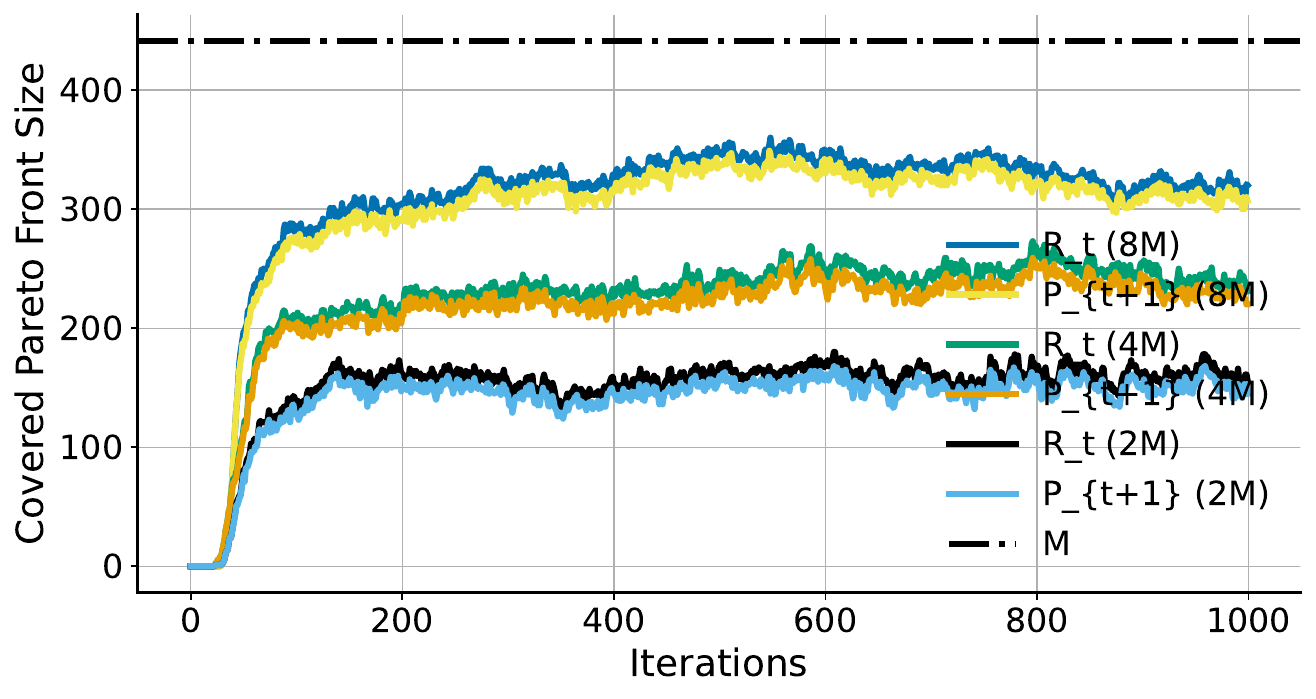}
        \caption{The number of distinct objective values on the Pareto front of the \NSGA (\Cref{nsga}) with standard bit mutation and fair selection for the combined parent and offspring population ($R_t$) and the selected population ($P_{t+1}$) over $1000$ iterations.
            The algorithm is run on the $4$-LOTZ problem of size $n = 40$, with the dashed line showing the total size of the Pareto front ($M = 441$) for reference.
            From top to bottom, each pair of curves refers, respectively, to the population size $N = 8M$, $N = 4M$, and $N = 2M$.
        }
        \label{fig:pareto_fronts_all}
    \end{minipage}
\end{figure}

\begin{figure}[t]
    \centering
    \includegraphics[width=0.95\columnwidth]{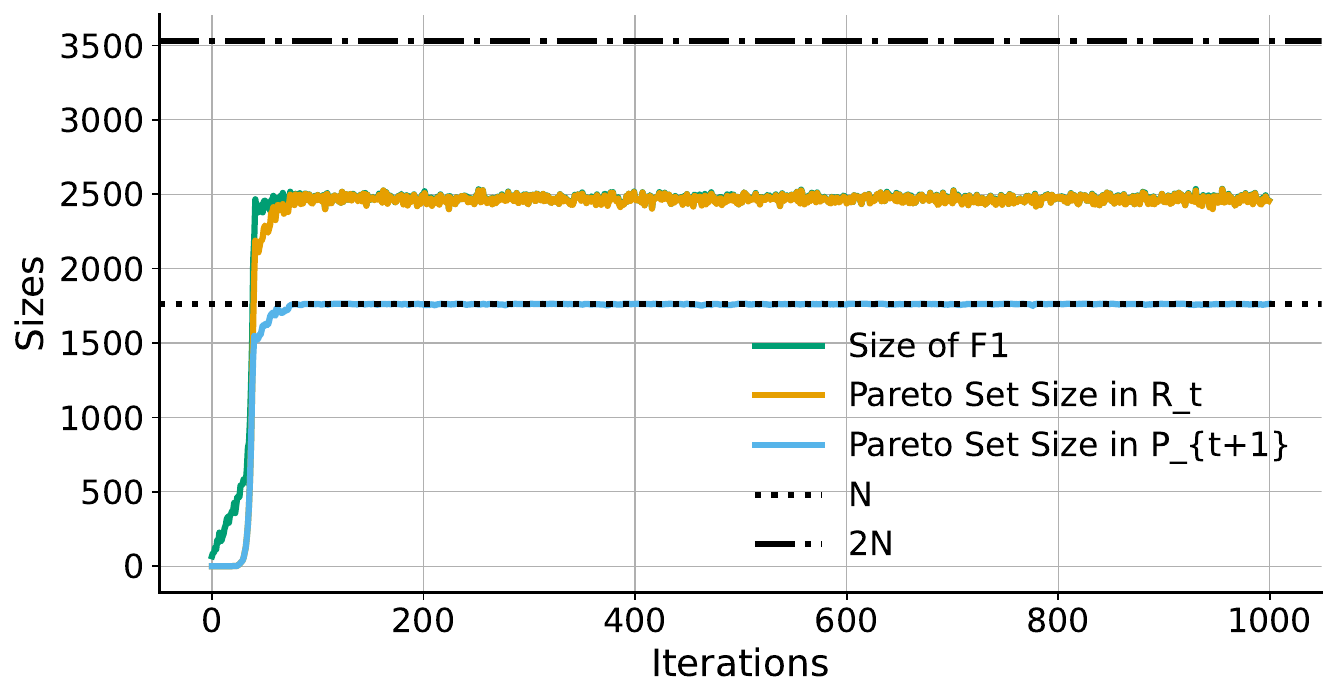}
    \caption{The number of individuals in the non-dominated subpopulation~$F_1$ as well as the number of Pareto-optimal individuals in the combined parent and offspring population~($R_t$) and the selected population~($P_{t + 1}$) of the \NSGA (\Cref{nsga}) with standard bit mutation and fair selection over $1000$ iterations.
        The algorithm is run on the $4$-LOTZ problem of size $n = 40$, resulting in a total size of the Pareto front of $M = 441$, with $N = 4M$. The dashed lines show the total population size of~$R_t$ (namely $2N$) and of~$P_{t + 1}$ (namely $N$) for reference.
        This is the same run as in \Cref{fig:pareto_fronts_all} for $N = 4M$.
    }
    \label{fig:F1_pareto_set_size_all}
\end{figure}

We study empirically how well the \NSGA approximates the Pareto front of $4$-LOTZ in three different settings, two of which study different algorithm variants that are not covered by our theoretical analysis.
Since it is costly to run the \NSGA, we limit our tests to the problem size $n = 40$, which results in a Pareto front of size $M = 441$.
We study the following algorithm variant settings:
\begin{enumerate}
    \item The \textbf{classic \NSGA} as seen in \Cref{nsga}, that is, the variant that uses fair selection when creating offspring (each parent generates exactly one offspring) as well as standard bit mutation.
          This setting aims to observe what fraction of the Pareto front the algorithm is still missing.
          Moreover, this setting acts as a baseline for the other settings.

    \item The \textbf{\NSGA with different parent selection mechanisms}.
          In addition to fair selection, we consider: (a) \emph{random selection}, where each offspring is generated by choosing a parent uniformly at random from the current parent population; (b) \emph{binary tournament selection}, where each offspring is generated by choosing the fittest of two parents, each chosen uniformly at random (breaking ties via non-dominated sorting, crowding distance, and then uniformly at random) with replacement from the current parent population; and (c) \emph{crossover}.
          The algorithm variant with crossover decides uniformly and independently at random for each offspring, before the parent selection, whether it uses crossover followed by mutation or whether it only uses mutation (with random parent selection).
          If it uses crossover, the algorithm first picks two parents uniformly at random, with replacement, and then performs uniform crossover among these two parents before applying mutation to the crossover result.
          All these variants still use standard bit mutation.

    \item The \textbf{\NSGA with one-bit mutation}, where an offspring is generated by flipping exactly one bit in the parent, with the bit position chosen uniformly at random.
          This variant still uses fair parent selection.
\end{enumerate}

\textbf{Classic NSGA-II.}
\Cref{fig:pareto_fronts_all} shows how well the \NSGA covers the Pareto front over a span of~$1000$ iterations, for three different population sizes.
We see that it is missing for each configuration still a constant fraction of the Pareto front, but this missing fraction is smaller the larger the population size is.
Furthermore, we see that while the combined population~$R_t$ covers more of the Pareto front then the selected population~$P_{t + 1}$, the difference is not very large.
This suggests that the offspring population does not add many new objective values such that the selection process for the new parent population is able to remove sufficiently many objective values for the process to converge prematurely.
This happens very consistently, as the trends of both curves are very similar.

\begin{figure}[t]
    \centering
    \includegraphics[width=0.95\columnwidth]{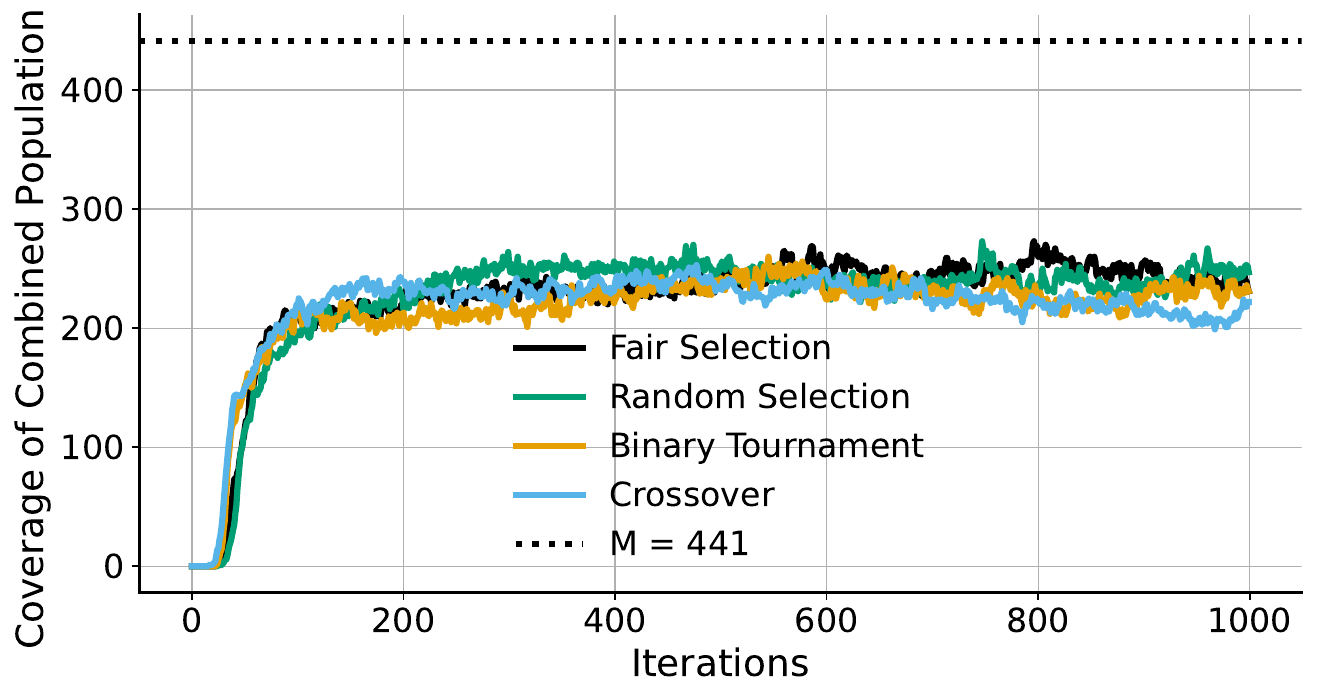}
    \caption{The number of distinct objective values on the Pareto front of the \NSGA (\Cref{nsga}) with standard bit mutation and different selection methods for the combined parent and offspring population ($R_t$) over $1000$ iterations.
        The algorithm is run on the $4$-LOTZ problem of size $n = 40$, with the dashed line showing the total size of the Pareto front ($M = 441$) for reference.
        The population size of the algorithm is $N = 4M$.
        The data for the run of the \emph{Standard} algorithm is the same as in \Cref{fig:pareto_fronts_all} with $N = 4M$.
    }
    \label{fig:coverage_combined_population}
\end{figure}

\begin{figure}[t]
    \centering
    \includegraphics[width=0.95\columnwidth]{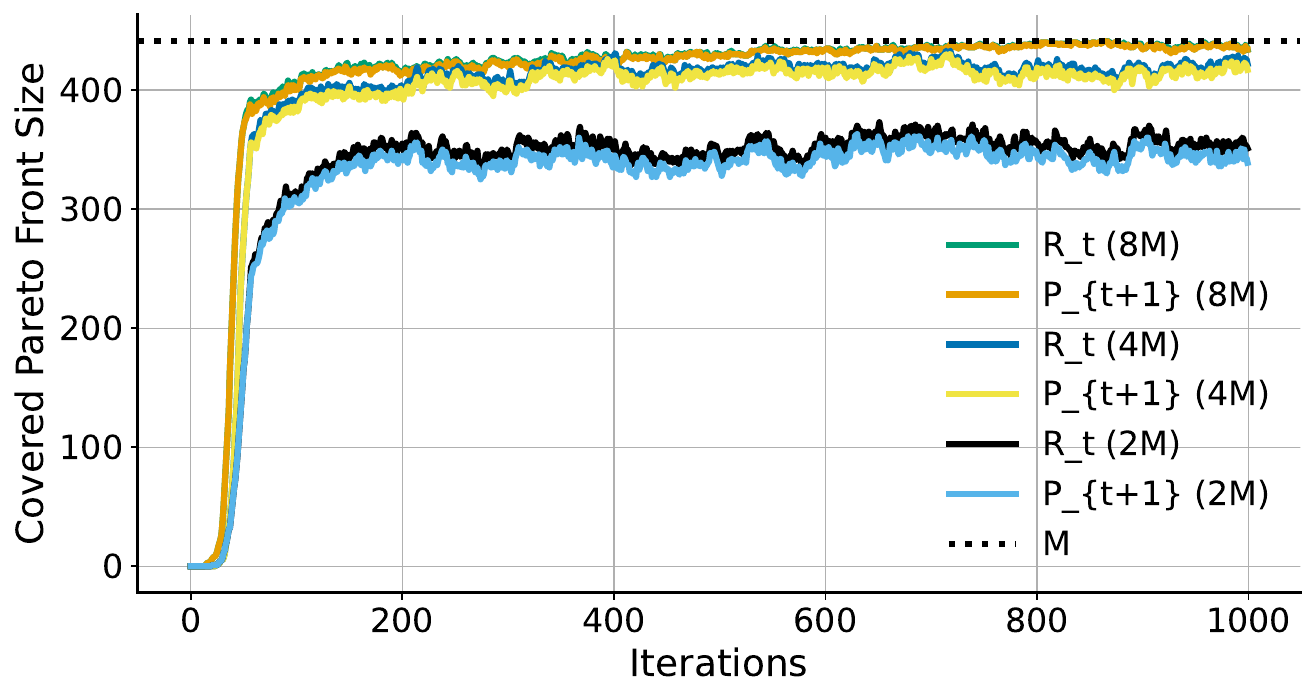}
    \caption{The number of distinct objective values on the Pareto front of the \NSGA (\Cref{nsga}) with one-bit mutation and fair selection for the combined parent and offspring population ($R_t$) and the selected population ($P_{t+1}$) over $1000$ iterations.
        The algorithm is run on the $4$-LOTZ problem of size $n = 40$, with the dashed line showing the total size of the Pareto front ($M = 441$) for reference.
        From top to bottom, each pair of curves refers, respectively, to the population size $N = 8M$, $N = 4M$, and $N = 2M$.
    }
    \label{fig:results_A_B}
\end{figure}

\Cref{fig:F1_pareto_set_size_all} shows how many individuals are Pareto-optimal over~$1000$ iterations, for three different population sizes.
We see that in all cases, the selected population~$P_{t + 1}$ for the next iteration consists almost exclusively of Pareto-optimal individuals.
The same is not true for the combined population, as it is easy to create non-optimal individuals.
Nonetheless, almost all of the non-dominated individuals are also Pareto-optimal.
Still, non-dominated individuals that are not Pareto-optimal regularly result in~$P_{t + 1}$ containing a small number of non-optimal individuals.

\textbf{Different parent selection mechanisms.}
\Cref{fig:coverage_combined_population} shows the Pareto front coverage of the combined population of the \NSGA algorithm with different selection mechanisms, all with population size $N = 4M$, measured over $1000$ iterations.
As we prove with \Cref{thm:general} for all mechanisms but crossover, the algorithm struggles.
The behavior seems to be very similar in all configurations, suggesting that the exact type of selection mechanism is not very relevant (recalling that \Cref{thm:general} covers an even more general class).
Notably, crossover exhibits a very similar performance, further suggesting that various mechanisms fail to make the \NSGA cover the entire Pareto front efficiently.

\textbf{One-bit mutation.}
\Cref{fig:results_A_B} shows how well the \NSGA with one-bit mutation covers the Pareto front over a time span of $1000$ iterations, for three different population sizes.
We see that, similar to \Cref{fig:pareto_fronts_all} with standard bit mutation, the algorithm typically does not cover the entire Pareto front during this time, and it covers a larger fraction when using a larger population size.
However, the fraction covered is larger as for the results seen in \Cref{fig:pareto_fronts_all}.
For a population size of $N = 8M$, the algorithm even occasionally covers the entire Pareto front, but it is missing, on average,~$5$ objective values on the Pareto front starting from around iteration~$800$, showing that it does not \emph{converge} to the entire Pareto front.
This improved coverage can be potentially attributed to the fact that one-bit flip mutation never creates a copy of the parent.
However, the main difference between \Cref{fig:results_A_B,fig:pareto_fronts_all} seems to be that the \NSGA with one-bit mutation creates a more diverse population early on.
Once the algorithm struggles to make progress, the relation between~$R_t$ and~$P_{t + 1}$ is comparable to that seen for standard bit mutation in \Cref{fig:pareto_fronts_all}.
This suggests that not creating copies of parents may help when not too many solutions are on the Pareto front and there is enough space to create new, incomparable solutions.

\section{Conclusion}
\label{sec:conclusion}

In this work, we have extended the result of~\cite{ZhengD24many}, which shows that the \NSGA has enormous difficulties solving the many-objective \omm problem, to the \lotz benchmark. This result is interesting, and technically non-trivial, as it shows that the difficulties observed in~\cite{ZhengD24many} persist also in situations where non-dominated sorting is relevant for the selection, and thus could potentially overcome the weaknesses of the crowding distance pointed out in~\cite{ZhengD24many}.
Together with the recent results showing that the NSGA-III, SMS-EMOA, and SPEA2 do not have any such weaknesses for many objectives~\cite{WiethegerD23,OprisDNS24,ZhengD24,WiethegerD24,RenBLQ24}, our result suggests that the crowding distance as secondary selection criterion is less suitable compared to several alternatives proposed in the literature.

For future work, it would be interesting to study other variants of the \NSGA rigorously, for example, when using one-bit mutation or crossover.

\section*{Acknowledgments}

This research benefited from the support of the FMJH Program Gaspard Monge for optimization and operations research and their interactions with data science.

\bibliographystyle{alphaurl}
\bibliography{alles_ea_master,ich_master,rest}

\end{document}